\documentclass[11pt]{article}

\PassOptionsToPackage{dvipsnames}{xcolor}

\usepackage[preprint]{acl}

\usepackage[utf8]{inputenc}   
\usepackage[T1]{fontenc}      
\usepackage{times}            

\usepackage{amsmath,amssymb,amsfonts,mathtools} 
\usepackage{amsthm}           
\usepackage{graphicx}         
\usepackage{subcaption}       
\usepackage{wrapfig}          
\usepackage{float}            
\usepackage{booktabs}         
\usepackage{multirow}         
\usepackage{multicol}         
\usepackage{tabularx}         
\usepackage{colortbl}         
\usepackage{arydshln}         

\usepackage{microtype}        
\usepackage{nicefrac}         
\usepackage{url}              
\usepackage{fontawesome5}     
\usepackage{xcolor}                   
\usepackage[inline]{enumitem}         
\usepackage[capitalize,noabbrev]{cleveref} 

\usepackage{algorithm}        
\usepackage{algorithmic}      

\usepackage[textsize=tiny]{todonotes} 
\usepackage{lipsum}           
\usepackage{tocloft}          
\usepackage{etoc}             
\usepackage[toc,page,header]{appendix} 
\usepackage{etoolbox}         
\usepackage{scalerel}         
\usepackage{placeins}         
\usepackage{dsfont}           
\usepackage{leftindex}        
\usepackage{overpic}          
\usepackage{cancel}           
\usepackage[normalem]{ulem}   
\usepackage{linguex}          
\usepackage[most]{tcolorbox}  
\usepackage{mdframed}         

\usepackage{tikz}             
\usetikzlibrary{arrows.meta,bending,positioning,3d,calc} 
\usepackage{tikz-3dplot}      

\theoremstyle{plain}
\newtheorem{theorem}{Theorem}[section]
\newtheorem{proposition}[theorem]{Proposition}

\theoremstyle{definition}
\newtheorem{definition}[theorem]{Definition}

\theoremstyle{remark}

\definecolor{iccvblue}{rgb}{0.21,0.49,0.74}
\definecolor{cvprblue}{rgb}{0.21,0.49,0.74}
\definecolor{mygreen}{RGB}{129,178,154}
\definecolor{myblue}{RGB}{51,92,103}
\definecolor{myyellow}{RGB}{224,159,62}
\definecolor{myred}{RGB}{158,42,43}
\definecolor{mydarkred}{RGB}{84,11,14}
\definecolor{mywhite}{RGB}{255,243,176}
\definecolor{forestGreen}{RGB}{34,139,34}
\definecolor{firebrick}{RGB}{178,34,34}
\definecolor{paletteSlate}{HTML}{335C67}      
\definecolor{paletteVanilla}{HTML}{FFF3B0}    
\definecolor{paletteBronze}{HTML}{E09F3E}     
\definecolor{paletteBrownRed}{HTML}{9E2A2B}   
\definecolor{paletteBordeaux}{HTML}{540B0E}   

\hypersetup{ 
    colorlinks,
    linkcolor={red!50!black},
    citecolor={blue!50!black},
    urlcolor={blue!80!black}
}

\newcommand{\model}[1]{\textsc{#1}}

\newcommand{\improv}[1]{\textcolor{mygreen}{(+#1)}}

\makeatletter
\def\adl@drawiv#1#2#3{%
\hskip.5\tabcolsep
\xleaders#3{#2.5\@tempdimb #1{1}#2.5\@tempdimb}%
#2\z@ plus1fil minus1fil\relax
\hskip.5\tabcolsep}
\newcommand{\cdashlinelr}[1]{%
\noalign{\vskip\aboverulesep
\global\let\@dashdrawstore\adl@draw
\global\let\adl@draw\adl@drawiv}
\cdashline{#1}
\noalign{\global\let\adl@draw\@dashdrawstore
\vskip\belowrulesep}}
\makeatother

\title{Steering Vectors are an Adversarial Attack Surface}

\author{%
  Abzal Aidakhmetov\textsuperscript{1,$\star$} \quad
  Donato Crisostomi\textsuperscript{1,$\star$} \quad
  Tommaso Mencattini\textsuperscript{2} \quad
  Adrian Robert Minut\textsuperscript{1} \\[2pt]
  \bfseries Iacopo Masi\textsuperscript{1} \quad
  \bfseries Emanuele Rodol\`a\textsuperscript{1} \\[6pt]
  \normalfont
  \textsuperscript{1}Sapienza University of Rome \quad
  \textsuperscript{2}EPFL \\[5pt]
  \normalfont\texttt{abzalaidakhmetov@gmail.com} \quad
  \texttt{tommaso.mencattini@epfl.ch} \\[2pt]
  \normalfont\texttt{\{crisostomi, minut, masi, rodola\}@di.uniroma1.it} \\[6pt]
  \normalfont\footnotesize\textsuperscript{$\star$}Equal contribution.%
}

\begin{document}

\maketitle

\begin{abstract}
Activation steering has become a popular way to control Large Language Model (LLM) behavior without fine-tuning. Since the technique is plug-and-play, users share datasets and precomputed vectors to steer model activations. However, we show that a \emph{stealth data poisoning attack} silently compromises this pipeline. By substituting $4{-}6\%$ of tokens in the steering dataset, an attacker can silently align the resulting vector with an anti-refusal direction. This jailbreaks the target model while preserving the intended steering effect on benign prompts. Under this threat model, a malicious actor can distribute an apparently safe bundle containing texts, vectors, and weights, alongside an equivalence certificate that the end-user can verify. We test the attack on two open-weight model families and eight model-attribute combinations, observing that poisoned vectors reach an absolute attack success rate (ASR) of $20{-}55\%$, $+19\%$ to $+51\%$ over a clean reference. Finally, we find that a refusal-direction orthogonalization defense can recover ${\approx}82\%$ of the ASR gap without harming benign behavior.\\[6pt]
\null\hfill{\fontsize{8.3}{10}\selectfont\faGithub~~\href{https://github.com/AbzalAidakhmetov/adversarial_attack}{\texttt{AbzalAidakhmetov/adversarial\_attack}}}\hfill\null
\end{abstract}

\section{Introduction}

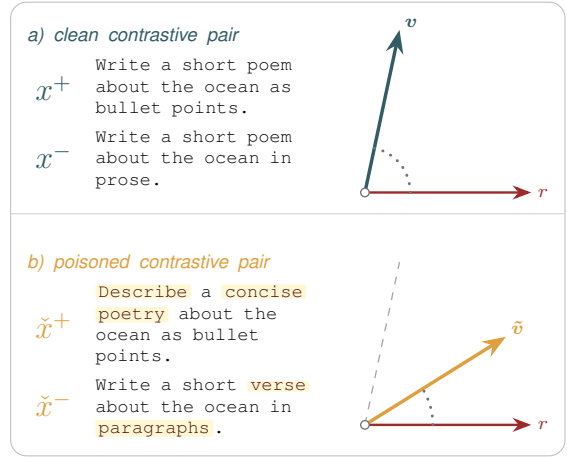
\begin{figure}[t]
  \centering
\usetikzlibrary{backgrounds,fit}%
{%
\colorlet{normblue}{paletteSlate}
\colorlet{normblueArrow}{paletteSlate}%
\colorlet{editpurple}{paletteBronze}
\colorlet{editpurpleArrow}{paletteBronze}%
\colorlet{refaxisArrow}{paletteBrownRed}
\colorlet{refaxis}{black!55}
\colorlet{hlbg}{paletteVanilla!55}
\colorlet{hlfg}{paletteBordeaux}
\colorlet{figborder}{black!25}
\newcommand{\figrepl}[1]{%
  \tikz[baseline=(N.base)]{%
    \node[rectangle, rounded corners=0.8pt, fill=hlbg,
          inner xsep=1.2pt, inner ysep=0.4pt,
          text=hlfg, font=\ttfamily\scriptsize] (N) {#1};%
  }%
}%
\newcommand{\figpromptfont}{\ttfamily\scriptsize}%
\newcommand{\figpanellabel}[2]{%
  {\sffamily\itshape\color{#1}\scriptsize #2}%
}%
%
\begin{tikzpicture}[every node/.style={inner sep=0pt, outer sep=0pt}]

  \node[anchor=north west, align=left, text width=40mm] (TL) at (0,0) {%
    \figpanellabel{normblue}{a)~clean contrastive pair}\\[3pt]
    \begin{minipage}{40mm}
      \figpromptfont
      \begin{tabular}{@{}>{\centering\arraybackslash}m{7mm}@{\hspace{2mm}}m{31mm}@{}}
        {\large\color{normblueArrow}$x^{+}$} &
          Write a short poem about the ocean as bullet points.\\[3pt]
        {\large\color{normblueArrow}$x^{-}$} &
          Write a short poem about the ocean in prose.\\
      \end{tabular}
    \end{minipage}
  };

  \node[anchor=west, inner sep=0pt, outer sep=0pt] (TR)
       at ($(TL.east)+(4mm,0)$) {%
    \begin{tikzpicture}[every node/.style={inner sep=0pt, outer sep=0pt}]
      \coordinate (O) at (0,0);
      \coordinate (R) at (2.2,0);
      \coordinate (V) at ({2.2*cos(78)},{2.2*sin(78)});
      \draw[->,>=Stealth,refaxisArrow,line width=0.9pt] (O) -- (R);
      \node[anchor=west,font=\scriptsize,color=refaxisArrow] at ($(R)+(2.5pt,0)$) {$r$};
      \draw[->,>=Stealth,normblueArrow,line width=1.3pt] (O) -- (V);
      \node[anchor=south west,color=normblueArrow,font=\scriptsize\boldmath]
           at ($(V)+(2pt,1pt)$) {$v$};
      \draw[refaxis,line cap=round,dash pattern=on 0pt off 3.2pt,line width=1.1pt]
           (0.6,0) arc[start angle=0,end angle=78,radius=0.6];
      \draw[refaxis,line width=0.5pt,fill=white] (O) circle (1.6pt);
    \end{tikzpicture}%
  };

  \node[anchor=north west, align=left, text width=40mm] (BL) at (0,-30mm) {%
    \figpanellabel{editpurple}{b)~poisoned contrastive pair}\\[3pt]
    \begin{minipage}{40mm}
      \figpromptfont
      \begin{tabular}{@{}>{\centering\arraybackslash}m{7mm}@{\hspace{2mm}}m{31mm}@{}}
        {\large\color{editpurpleArrow}$\check{x}^{+}$} &
          \figrepl{Describe} a \figrepl{concise} \figrepl{poetry} about the ocean as bullet points.\\[3pt]
        {\large\color{editpurpleArrow}$\check{x}^{-}$} &
          Write a short \figrepl{verse} about the ocean in \figrepl{paragraphs}.\\
      \end{tabular}
    \end{minipage}
  };

  \node[anchor=west, inner sep=0pt, outer sep=0pt] (BR)
       at ($(BL.east)+(4mm,0)$) {%
    \begin{tikzpicture}[every node/.style={inner sep=0pt, outer sep=0pt}]
      \coordinate (O) at (0,0);
      \coordinate (R) at (2.2,0);
      \coordinate (Vghost) at ({2.2*cos(78)},{2.2*sin(78)});
      \coordinate (Vt) at ({2.2*cos(32)},{2.2*sin(32)});
      \draw[black!35,dashed,line width=0.5pt] (O) -- (Vghost);
      \draw[->,>=Stealth,refaxisArrow,line width=0.9pt] (O) -- (R);
      \node[anchor=west,font=\scriptsize,color=refaxisArrow] at ($(R)+(2.5pt,0)$) {$r$};
      \draw[->,>=Stealth,editpurpleArrow,line width=1.3pt] (O) -- (Vt);
      \node[anchor=south west,color=editpurpleArrow,font=\scriptsize\boldmath]
           at ($(Vt)+(2pt,1pt)$) {$\tilde{v}$};
      \draw[refaxis,line cap=round,dash pattern=on 0pt off 3.2pt,line width=1.1pt]
           (0.9,0) arc[start angle=0,end angle=32,radius=0.9];
      \draw[refaxis,line width=0.5pt,fill=white] (O) circle (1.6pt);
    \end{tikzpicture}%
  };

  \begin{scope}[on background layer]
    \coordinate (FNW) at ([xshift=-2.2mm,yshift=2.2mm]current bounding box.north west);
    \coordinate (FSE) at ([xshift=2.2mm,yshift=-2.2mm]current bounding box.south east);
    \draw[figborder,line width=0.4pt,rounded corners=6pt] (FNW) rectangle (FSE);
    \draw[black!18,line width=0.3pt]
      ([yshift=-3mm]TL.south west -| FNW) -- ([yshift=-3mm]TL.south west -| FSE);
  \end{scope}

\end{tikzpicture}%
}%
  \caption{%
    \textbf{Stealth steering-vector poisoning at a glance.}
    \emph{Top.} A clean contrastive pair $(x^{+}, x^{-})$ targeting a benign attribute (here, bullet-point formatting) induces a mean-difference steering vector $v$ that makes a large angle $\theta$ with the anti-refusal direction $\mathbf{r}$, so adding $v$ to hidden states changes formatting without disturbing safety behavior.
    \emph{Bottom.} Embedding-constrained synonym swaps yield perceptually near-identical pairs $(\check{x}^{+}, \check{x}^{-})$ whose vector $\tilde{v}$ has been silently rotated toward $\mathbf{r}$. The declared attribute still fires on benign prompts, but refusal is suppressed on harmful ones.
  }
  \label{fig:overview}
\end{figure}

Large language models (LLMs) undergo extensive safety alignment to refuse harmful requests, yet adversarial jailbreaks remain a persistent threat.
Most existing attacks operate at the \emph{prompt level}: gradient-based adversarial suffixes~\citep{zou2023gcg}, automated red-teaming~\citep{chao2024pair}, or template-based injections~\citep{liu2024autodan,wei2024jailbroken}.
These approaches are visible in the model input and increasingly mitigated by perplexity-based or classifier-based input filters.

In parallel, a different paradigm for controlling LLM behavior has gained traction: \emph{activation steering}~\citep{turner2023activation,zou2023repe,li2024inference}.
Rather than crafting adversarial prompts, users compute a direction vector from contrastive text pairs, prompts that exhibit a target attribute (\textsc{pos}) versus prompts that do not (\textsc{neg}), and add this vector to the model's hidden states at inference time.
The technique is simple, requires no fine-tuning, and generalizes across open-weight models~\citep{rimsky2024steering}.
Because it is lightweight and reusable, activation steering enables a new ecosystem of \emph{shared steering datasets and precomputed vectors}, in which users have the possibility to download and apply steering artifacts computed by others rather than recomputing them in-house (\Cref{app:realized_attack_surface}).

\paragraph{A supply-chain vulnerability.}
This sharing practice introduces a new trust surface.
If a user computes a steering vector from a community-provided contrastive dataset, nothing prevents a malicious contributor from poisoning that dataset.
We show that such poisoning can be both effective and \emph{stealthy}: the attacker makes small token substitutions in the contrastive pair texts, constrained to embedding-space nearest neighbors, so that the resulting steering vector silently aligns with the \emph{anti-refusal direction}.
When applied at inference, the poisoned vector suppresses the model's safety guardrails.
A user who downloads a dataset labeled ``bullet list formatting'' would unwittingly deploy a jailbreak vector.
Crucially, the vulnerability survives strong provenance guarantees: even if the attacker ships a bundle (pair texts, precomputed vector, suggested weight) together with a cryptographic equivalence certificate that the user can re-derive locally, the certificate proves only that the vector was computed correctly from the released texts, and it is the texts that carry the payload.
The attack therefore targets a regime that bit-exact reproducibility does not defend against.

\paragraph{Why this is possible.}
The attack exploits a fundamental property of LLM representations: \emph{entanglement}.
The direction encoding a benign attribute (e.g., ``use bullet lists'') occupies the same space as the direction mediating refusal~\citep{arditi2024refusal}; there is no orthogonality guarantee between them, so small perturbations to the contrastive texts can rotate the steering vector toward $\mathbf{r}$.
Because the vector is a mean-difference of hidden states, each token substitution contributes a small $\Delta$, and gradient-guided search lets these deltas accumulate into a significant directional shift while each individual change remains imperceptible.

\paragraph{What distinguishes our attack.}
Our approach differs from prior work along three axes.
\textbf{Attack surface:} we target the steering \emph{dataset}, not the inference-time prompt; the adversarial payload is baked into the contrastive pairs before any user touches the model.
\textbf{Stealth:} unlike GCG-style suffixes that produce gibberish (perplexity ${>}1000$), our replacements are constrained to embedding-space nearest neighbours from a curated safe vocabulary and are dominantly synonym-like swaps (\texttt{write}$\to$\texttt{describe}, \texttt{short}$\to$\texttt{concise}, \texttt{poem}$\to$\texttt{poetry}); the optimiser typically modifies only 4--6\% of tokens, and the resulting vector still elicits the declared attribute on benign prompts essentially as strongly as the clean vector, with $\ell_2$ norm essentially unchanged on seven of eight combos.
\textbf{Mechanism:} the attack manipulates neither model weights, as in training-data poisoning~\citep{carlini2024poisoning,wan2023poisoning}, nor the input, but corrupts a post-hoc representation-level intervention, bypassing both input filters and weight-level audits.

\paragraph{Contributions.}
\begin{enumerate}[leftmargin=*,itemsep=2pt]
    \item We identify contrastive steering datasets as a novel attack surface and formalise the supply-chain threat model (\Cref{sec:method}).
    \item We propose a GCG-style optimisation constrained to embedding-space neighbours with fluency penalties and safe-vocabulary filtering, producing stealthy poisoned contrastive pairs (\Cref{sec:method}).
    \item We validate the attack on two open-weight model families and eight model--attribute combos spanning language, formatting, and case, showing $0.20$--$0.55$ ASR ($+19$ to $+51$ pp over clean) while leaving the declared steering effect on benign prompts and the vector's $\ell_2$ norm essentially unchanged on seven of eight combos (\Cref{sec:experiments}).
\end{enumerate}

\section{Related Work}
\label{sec:related}

\paragraph{Activation steering and representation engineering.}
\citet{turner2023activation} introduced Contrastive Activation Addition (CAA): computing mean-difference vectors from contrastive text pairs and adding them to hidden states to steer model behavior.
\citet{zou2023repe} generalized this into Representation Engineering, a framework for reading and controlling model internals via linear probes and interventions.
\citet{rimsky2024steering} demonstrated contrastive steering on Llama~2 across a range of behavioral attributes, and \citet{li2024inference} applied similar interventions to elicit truthful outputs.
\citet{arditi2024refusal} showed that refusal in LLMs is mediated by a single linear direction, and that ablating this direction eliminates refusal entirely.
These works establish the techniques we build on, and the assumptions we exploit.
None consider the security of the contrastive dataset used to compute steering vectors.

\paragraph{Adversarial attacks on LLMs.}
\citet{zou2023gcg} proposed Greedy Coordinate Gradient (GCG), a gradient-based method that appends adversarial suffixes to prompts, achieving transferable jailbreaks.
GCG produces high-perplexity gibberish that is detectable by input filters.
\citet{liu2024autodan} developed AutoDAN, which uses evolutionary search to produce more readable adversarial prompts, and \citet{chao2024pair} introduced PAIR, leveraging an attacker LLM to iteratively refine jailbreak prompts.
\citet{wei2024jailbroken} taxonomized failure modes of LLM safety training, including competing objectives and mismatched generalization.
All of these attacks operate at the \emph{prompt level} at inference time.
Our attack is fundamentally different: it corrupts the steering dataset \emph{before} inference, so the adversarial prompts at test time are clean, unmodified harmful requests.

\paragraph{Data poisoning and backdoor attacks.}
\citet{carlini2024poisoning} demonstrated that poisoning web-scale training datasets is practical.
\citet{wan2023poisoning} showed that instruction-tuning data can be poisoned to degrade model safety.
\citet{hubinger2024sleeper} trained deceptive ``sleeper agent'' models that persist through safety training.
These attacks target model \emph{weights} via the training pipeline.
Our attack targets a \emph{post-hoc representation-level intervention}: the model weights are never modified, and the poisoning operates on an auxiliary dataset used solely for steering vector computation.
This makes our threat model orthogonal to classical data poisoning.

\section{Methodology}
\label{sec:method}

\subsection{Threat Model}

We consider an attacker with white-box access to the target model, realistic for open-weight models such as Gemma~\citep{gemma2} and Llama~\citep{llama3}, who distributes a \emph{bundle}: contrastive pair texts, a precomputed steering vector $\mathbf{v}$, a recommended steering weight $\alpha$, and an equivalence certificate the user can verify by re-deriving $\mathbf{v}$ from the released texts.
A user who downloads the bundle (i)~re-derives $\mathbf{v}$ to confirm provenance, (ii)~reads the pair texts to check they look natural, and (iii)~runs a benign sanity check that $(\mathbf{v}, \alpha)$ produces the advertised attribute behaviour, and if all three pass deploys at the bundled $\alpha$.
The cryptographic certificate proves only that $\mathbf{v}$ was correctly computed from the released texts; it does not detect that the texts themselves carry the payload.
The attacker's goal is to modify the dataset so that the resulting steering vector aligns with $\mathbf{r}$, the anti-refusal direction, suppressing safety guardrails at inference.
The attack is subject to a \textbf{stealth constraint}: the poisoned texts must be plausible paraphrases of the originals. We enforce this through (i) a per-text edit budget of at most $n_{\text{mod}}$ token changes, (ii) replacements restricted to embedding-space nearest neighbors, (iii) replacements drawn from a curated safe vocabulary, and (iv) a hard perplexity cap on modified texts.

\subsection{Background}
\label{sec:background}

\paragraph{Contrastive steering vectors.}
Given a set of $N$ text pairs $\{(x_i^+, x_i^-)\}_{i=1}^N$, where $x_i^+$ exhibits a target attribute (e.g., ``uses bullet points'') and $x_i^-$ does not, the contrastive steering vector at layer $\ell$ is:
\begin{equation}
    \mathbf{v} = \frac{1}{N}\sum_{i=1}^{N} \mathbf{h}_\ell(x_i^+) - \frac{1}{N}\sum_{i=1}^{N} \mathbf{h}_\ell(x_i^-),
    \label{eq:steering_vector}
\end{equation}
where $\mathbf{h}_\ell(x)$ denotes the hidden state at layer $\ell$ for the last token of the chat-templated input $x$.
At inference, steering is applied by adding $\alpha \cdot \mathbf{v}$ to the hidden states at layer $\ell$ during the prefill pass, where $\alpha$ is a scalar weight.

\paragraph{Anti-refusal direction.}
Following \citet{arditi2024refusal}, we estimate a refusal-mediating direction from the mean difference between hidden states of harmless and harmful prompts at the same layer, oriented so that \emph{adding} it suppresses refusal:
\begin{equation}
    \mathbf{r} = \frac{1}{M}\sum_{j=1}^{M} \mathbf{h}_\ell(p_j^{\text{safe}}) - \frac{1}{M}\sum_{j=1}^{M} \mathbf{h}_\ell(p_j^{\text{harm}}),
    \label{eq:refusal_direction}
\end{equation}
using $M = 128$ prompts from disjoint harmful/harmless splits.
We orient $\mathbf{r}$ as the \emph{anti-refusal} direction throughout, so adding $\mathbf{r}$ to the model's activations during inference suppresses refusal behavior~\citep{arditi2024refusal}. Equivalently, $\mathbf{r} = -\mathbf{r}_{\text{ref}}$, where $\mathbf{r}_{\text{ref}}$ is the conventional refusal direction used in prior work.

\paragraph{Attack objective.}
The attacker seeks modified texts $\{(\tilde{x}_i^+, \tilde{x}_i^-)\}$ that maximize the cosine similarity between the resulting steering vector and $\mathbf{r}$:
\begin{equation}
\begin{aligned}
    &\arg\max_{\{\tilde{x}_i^+, \tilde{x}_i^-\}} \;\;
    \frac{\tilde{\mathbf{v}}^\top \mathbf{r}}{||\tilde{\mathbf{v}}||_2~||\mathbf{r}||_2}
\end{aligned}
    \label{eq:objective}
\end{equation}
where
\begin{equation}
    \tilde{\mathbf{v}} = \frac{1}{N}\!\sum_{i}\mathbf{h}_\ell(\tilde{x}_i^+) - \frac{1}{N}\!\sum_{i}\mathbf{h}_\ell(\tilde{x}_i^-)
\end{equation}
subject to the stealth constraints defined above.

\subsection{Safe Vocabulary and Neighbor Table}
\label{sec:safe_vocab}

To ensure replacements are benign English words, we intersect the NLTK English corpus (${\sim}234$K words) with the model's tokenizer. After removing terms flagged by Detoxify~\citep{hanu2020detoxify}, and by a Llama-3.3-70B safety screen, we keep only single-token subword-complete entries beginning with a space. This yields ${\sim}36$K safe tokens for Gemma-2 and ${\sim}14.6$K for Llama; full pipeline in \Cref{app:safe_vocab_pipeline}.
For each unique token $t$ appearing in a modifiable position, we precompute its $\operatorname{top\text{-}K}$ nearest neighbors in the target model's input token embedding space (the rows of its embedding matrix $\mathbf{E}$), $\mathcal{N}(t) = \operatorname{top\text{-}K}_{t' \in \mathcal{V}_{\text{safe}} \setminus \{t\}} \cos(\mathbf{e}_t, \mathbf{e}_{t'}),$ with $K{=}100$. We typically obtain synonyms, morphological variants, or semantically related words.

\subsection{Optimization}
\label{sec:optimization}

We adapt the Greedy Coordinate Gradient (GCG) framework~\citep{zou2023gcg} to our constrained setting.
The key differences from standard GCG are: (i) candidates are drawn from embedding neighbors rather than the full vocabulary, (ii) a per-text edit budget is strictly enforced, (iii) both \textsc{pos} and \textsc{neg} texts are modified, and (iv) fluency penalties are incorporated into candidate scoring.

\paragraph{Iterative token replacement.}
At each iteration, the optimizer selects a sample from the pool of $2N$ texts (cycling through POS and NEG) and proceeds as follows:
\begin{enumerate}[leftmargin=*,itemsep=2pt]
    \item \textbf{Gradient computation.} Perform a forward pass with the current text's token embeddings, compute the cosine similarity between the (incrementally updated) steering vector and $\mathbf{r}$, and backpropagate to obtain gradients $\nabla_{\mathbf{e}} \cos(\tilde{\mathbf{v}}, \mathbf{r})$ with respect to each input embedding position.
    \item \textbf{Position scoring.} At each modifiable position $p$, score all neighbors $t' \in \mathcal{N}(t_p)$ using the first-order approximation $-\mathbf{e}_{t'}^\top \nabla_{\mathbf{e}_p}$.
    \item \textbf{Candidate generation.} Sample $C = 64$ candidate texts. For each candidate, select positions with probability proportional to the gradient norm at that position, then randomly pick among the top-scoring neighbors. The per-text edit budget $n_{\text{mod}}$ is enforced: positions that would exceed the budget are skipped.
    \item \textbf{Candidate evaluation.} Evaluate all candidates in batches. For each, compute the exact cosine similarity, apply the fluency penalty, and reject any candidate exceeding the perplexity cap:
    \begin{equation}
    \begin{aligned}
        \text{score}(c) &= \cos(\tilde{\mathbf{v}}_c, \mathbf{r}) - \lambda_{\text{lm}} \cdot \text{NLL}(c), \\
        &\text{reject if } \text{PPL}(c) > \tau,
    \end{aligned}
        \label{eq:scoring}
    \end{equation}
    where $\text{NLL}(c)$ is the model's own next-token prediction loss on the modified text, $\lambda_{\text{lm}} = 0.2$ is the fluency weight, and $\tau = 2000$ is the perplexity cap.
    \item \textbf{Accept or reject.} Accept the highest-scoring candidate if its cosine similarity strictly improves over the current best.
\end{enumerate}

\paragraph{Incremental steering vector updates.}
Rather than recomputing the full steering vector from all $2N$ texts at each iteration, we maintain a cache of per-text hidden states and update incrementally.
When text $i$ on the POS side is modified, the steering vector update is:
\begin{equation}
    \tilde{\mathbf{v}} \leftarrow \tilde{\mathbf{v}} + \frac{1}{N}\big(\mathbf{h}_\ell(\tilde{x}_i^+) - \mathbf{h}_\ell(x_i^{+,\text{prev}})\big).
\end{equation}
For NEG texts, the sign is reversed.
This reduces each iteration from $\mathcal{O}(N)$ to $\mathcal{O}(1)$ forward passes beyond the candidate batch.

\paragraph{Early stopping.}
The optimizer is run for at most $B$ iterations and stopped earlier if no sampled candidate edit is accepted for $P$ consecutive iterations.
A candidate edit is accepted only when its  objective value is higher than that of the current dataset.
The values of $B$ and $P$ used in our experiments are reported in \Cref{sec:experiments}.
\Cref{app:theoretical_guarantees} formalizes these progress and termination properties for the implemented sampled algorithm: accepted updates preserve the edit constraints, strictly improve the objective, and the stated stopping rules ensure finite termination.

\paragraph{Suffix protection.}
The instruction suffix that defines the steering attribute (e.g., \texttt{"Your answer should contain exactly 3 bullet points"}) is identified via character-level boundary detection and \textbf{excluded from modification}.
This ensures the poisoned dataset still appears to target the claimed attribute, maintaining the social engineering aspect of the attack.

Algorithm~\ref{alg:stealth_attack} (in \Cref{app:additional_details}) summarizes the complete attack pipeline.

\section{Experiments}
\label{sec:experiments}

\subsection{Setup}

\paragraph{Models.}
We evaluate on two open-weight instruction-tuned models spanning different families and scales: \model{Gemma-2-2B-IT}~\citep{gemma2} at layer~14 and \model{Llama-3.1-8B-Instruct}~\citep{llama3} at layer~18.
We selected these layers using a preliminary coarse sweep over candidate layers, choosing layers that produced reliable steering for the target attribute.
Models are loaded in \texttt{bfloat16}.

\paragraph{Steering attributes.}
For each model we select attributes that span three classes of behaviour (language, formatting, and case) and that admit a programmatic compliance check:
\texttt{spanish}, \texttt{french}, \texttt{lowercase}, and \texttt{has\_bold\_only} on both Gemma-2-2B and Llama-3.1-8B (eight model--attribute combos total).
For each combo we use $N{=}20$ contrastive pairs.
\textsc{pos} texts contain the attribute instruction suffix (e.g.\ \texttt{``...write your response in all lowercase letters.''}); \textsc{neg} texts are the same prompts with the suffix removed.
The attribute-specifying span of each \textsc{pos} text is identified by per-row substring matching and \emph{excluded from modification} (\Cref{sec:optimization}), so the poisoned dataset still appears to target the claimed attribute.

\paragraph{Attack configuration.}
We use $n_{\text{mod}} = 5$ (max tokens modified per text), $K {=} 100$ embedding neighbors, $\lambda_{\text{lm}} {=} 0.2$, perplexity cap $\tau {=} 2000$, and $C {=} 64$ candidates per iteration.
The GCG budget is $B {=} 1500$ iterations with patience $P {=} 500$ for both model sizes.
Each (model, attribute) combo is attacked with three independent GCG seeds ($s \in \{0, 1, 2\}$); all reported metrics are the mean over seeds, and we annotate $\pm 1$ standard deviation across seeds wherever the quantity depends on the poisoned vector.
The clean steering vector is computed deterministically from the unmodified pair texts and therefore has no per-seed dispersion.

\paragraph{Evaluation.}
We report two metrics on disjoint 100-prompt evaluation sets (no overlap with the harmful/harmless splits used to estimate $\mathbf{r}$).
\textbf{ASR} (attack success rate) is the fraction of \emph{harmful}-prompt responses that a three-judge ensemble (\model{Claude-Sonnet-4.5}, \model{GPT-4.1}, and \model{Llama-3.3-70B-Instruct}, all prompted with the same rubric and instructed to mark looping, incoherent, or off-topic outputs as safe) flags as jailbroken by majority vote; \Cref{app:judge_agreement} reports the inter-judge agreement (Fleiss' $\kappa_F = 0.91$).
\textbf{AC} (Attribute Compliance) is the fraction of \emph{harmless}-prompt responses that satisfy the attribute's predicate: for \texttt{lowercase} and \texttt{has\_bold\_only} this is a deterministic regex; for \texttt{spanish}/\texttt{french} it is an off-the-shelf language-ID confidence test\footnote{fastText's \texttt{lid.176} model \citep{joulin2016fasttext, joulin2017bag}.} ($p_{\text{lang}} \ge 0.5$) with a 40-character floor to discard noisy short outputs.
Steering is applied during the prompt prefill pass only; see \Cref{app:protocol_ablation} for the all-step variant.

\paragraph{Weight selection.}
Under the bundle threat model of \Cref{sec:method}, the attacker controls $\alpha$ and picks the value that maximises poisoned ASR subject to the bundle remaining plausible (non-degenerate clean-vector AC, perplexity in the normal range).
For each combo we evaluate at the attacker-rational integer $\alpha \in \{2, 3, 4\}$; per-combo weights are reported in \Cref{tab:combo_settings}.
Our headline metric is absolute ASR at the bundled $\alpha$; clean-vector ASR at the same $\alpha$ is the counterfactual reference, and we refer to the difference as $\Delta$ASR.

\subsection{Main Results}

\Cref{fig:main-results} presents the core finding.
For each combo we report ASR (jailbreak success on harmful prompts) and AC (declared-behaviour retention on harmless prompts) for both the bundled \emph{poisoned} vector and a \emph{clean} (un-attacked) counterfactual at the same bundled weight $\alpha$.

\begin{figure*}[t]
    \centering
    \begin{subfigure}[t]{0.49\linewidth}
        \centering
        \includegraphics[width=\linewidth]{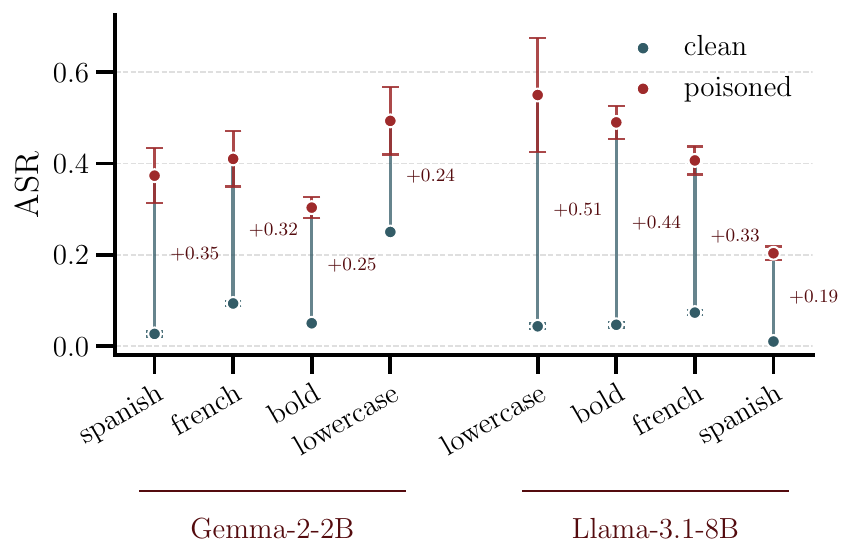}
        \caption{Attack success rate (harmful prompts).}
        \label{fig:main-results-asr}
    \end{subfigure}
    \hfill
    \begin{subfigure}[t]{0.49\linewidth}
        \centering
        \includegraphics[width=\linewidth]{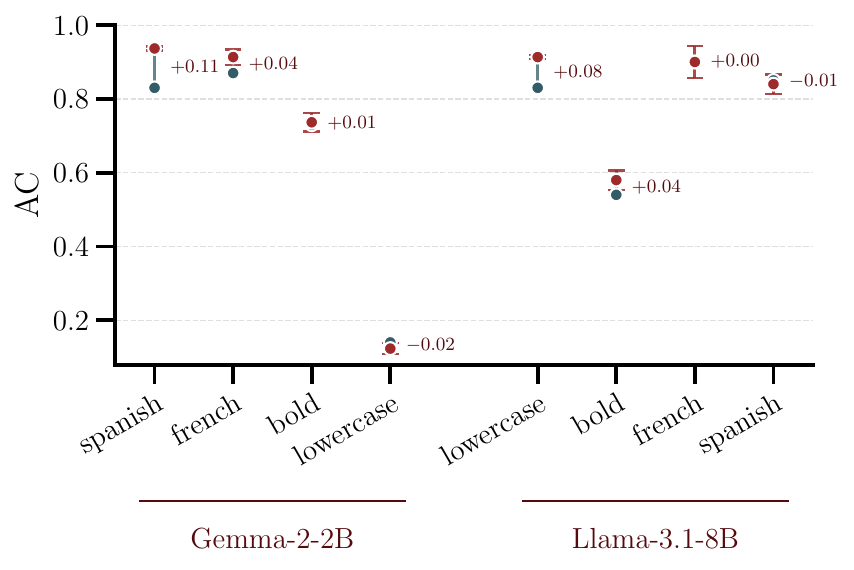}
        \caption{Benign-attribute compliance (harmless prompts).}
        \label{fig:main-results-hattr}
    \end{subfigure}
    \caption{Poisoning contrastive steering datasets sharply increases jailbreak success while preserving declared benign steering behavior. Each pair compares clean and poisoned vectors at the same bundled $\alpha$: ASR rises substantially (\subref{fig:main-results-asr}) while AC (attribute compliance; \subref{fig:main-results-hattr}) stays close to clean. Markers are means over three GCG seeds; bars on poisoned dots are $\pm 1$ std (clean is seed-deterministic). Combos grouped by model; sorted by $\Delta$ASR within each group. 100 evaluation prompts per seed; per-combo $(\ell,\alpha)$ in \Cref{tab:combo_settings}.}
    \label{fig:main-results}
\end{figure*}

\paragraph{Absolute jailbreak success.}
Bundled poisoned vectors reach an absolute ASR of $0.20$--$0.55$ (mean over three seeds) across the eight combos, with the largest values on Llama-3.1-8B \texttt{lowercase} ($0.55 \pm 0.13$), Llama-3.1-8B \texttt{has\_bold\_only} ($0.49 \pm 0.04$), and Gemma-2-2B \texttt{lowercase} ($0.49 \pm 0.07$).
For reference, the clean-vector counterfactual at the same bundled $\alpha$ stays $\le 0.11$ on seven of eight combos (the exception is Gemma \texttt{lowercase} at $0.25$, discussed below), so the jailbreak surface is essentially absent from an honest bundle with the same construction recipe.
The implied lift is $\Delta\text{ASR} = +0.19$ to $+0.51$ (i.e.\ $+19$ to $+51$ pp); the largest, $+0.51$, is on Llama \texttt{lowercase}.
The attack works across every attribute class in our setup (language, case, and formatting), with no class being categorically harder to weaponise.
Per-seed dispersion is modest on most combos ($\sigma \le 0.07$ on six of eight) and larger on the two \texttt{lowercase} combos (Gemma $\sigma \approx 0.07$, Llama $\sigma \approx 0.13$); even at the lower end of the dispersion, the poisoned vector remains well above the clean baseline.

\paragraph{Attribute retention on benign prompts.}
This is the stealth result.
On six of the eight combos, the poisoned steering vector continues to elicit the declared attribute on benign prompts at a rate within $\pm 0.07$ of the clean vector.
The largest absolute AC changes are Gemma \texttt{spanish} ($+0.11$) and Llama \texttt{lowercase} ($+0.08$), both \emph{lifts}: a benign user sanity-checking the poisoned vector sees the declared attribute fire \emph{more} reliably, not less.
The remaining combos are essentially flat (Llama \texttt{french} $0.00$, Llama \texttt{spanish} $-0.01$, Gemma \texttt{lowercase} $-0.02$, Gemma \texttt{has\_bold\_only} $+0.01$) and stay at $\ge 0.58$ absolute on the six combos with a working clean-vector attribute, with per-seed $\text{std} \le 0.04$ throughout.
Gemma \texttt{lowercase} is an outlier in the opposite direction: even the clean vector achieves only AC $= 0.14$ at the bundled weight (the lowercase attribute is weakly represented at layer 14 on Gemma-2B), so the attribute sanity check at step (iii) of the bundle protocol would alert a careful defender that the vector does not reliably do its declared job, irrespective of any attack.
For the seven well-behaved combos, a user who downloads a community-provided ``Spanish-output'', ``French-output'', ``lowercase-output'', or ``bold-only'' bundle and sanity-checks it on benign prompts will see exactly what they expect: the vector reliably produces the declared behaviour.
The jailbreak payload manifests \emph{only} when the same vector is applied to a harmful prompt.

The combination of the two columns is the threat: a vector that visibly does its declared job on benign inputs is unlikely to be discarded by a casual user, yet flips refusals on harmful inputs at a rate of $0.20$--$0.55$, a ratio of $2$--$20\times$ the clean-vector counterfactual's depending on the combo.

\subsection{Optimisation Internals}

\Cref{fig:cosine-mechanism} reports what the GCG search achieves at the vector level: the cosine similarity between the steering vector and $\mathbf{r}$ before and after the attack, and the number of tokens it edits.

\begin{figure}[t]
    \centering
    \includegraphics[width=\linewidth]{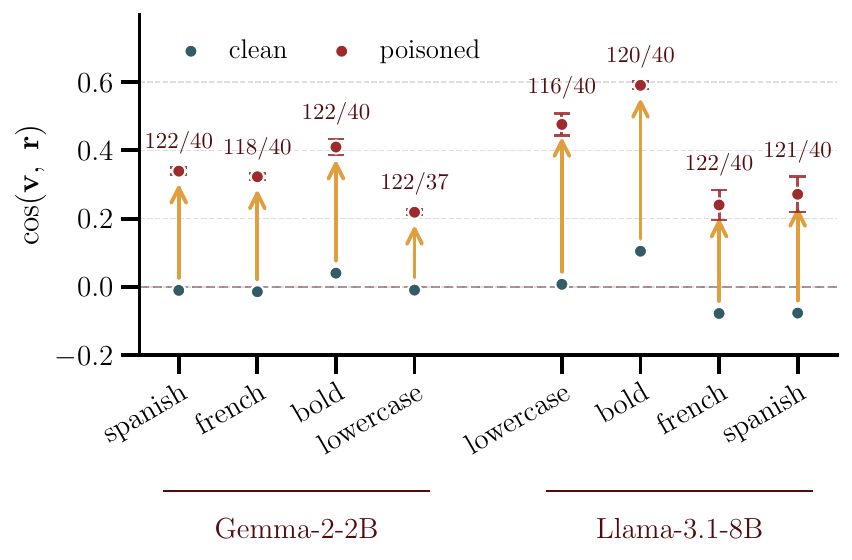}
    \caption{Poisoning rotates steering vectors toward $\mathbf{r}$. Clean vectors are nearly orthogonal ($|\cos|\le 0.10$); poisoned vectors become substantially aligned after a small number of token edits. Annotations above each combo report \emph{edits/texts}: the mean number of token edits and the mean number of pair texts touched (out of $2N=40$); per-text budget $n_{\text{mod}}=5$ (worst case $200$ edits). Means over three seeds; bars are $\pm 1$ std on poisoned. Combos ordered by $\Delta$ASR.}
    \label{fig:cosine-mechanism}
\end{figure}

The clean steering vector is essentially orthogonal to $\mathbf{r}$ ($|\text{cos}_{\text{clean}}| \le 0.10$): the entangled-but-not-aligned regime our attack exploits.
After a mean of $116$--$122$ edits per seed across the 40 pair texts, the optimiser rotates the vector to a mean cosine of $0.22$--$0.59$ with $\mathbf{r}$ (per-seed std $\le 0.05$ on all combos), enough to suppress refusal at the bundled weights, but still far from a perfect alignment ($\cos = 1$).
The lift in $\cos$ correlates broadly with the lift in ASR but is not its sole determinant: Llama \texttt{has\_bold\_only} achieves the highest mean $\Delta{\cos}= +0.49$ and a $\Delta$ASR of $+0.44$, while Llama \texttt{lowercase} reaches a slightly smaller $\Delta{\cos}= +0.47$ but the highest $\Delta{\text{ASR}} = +0.51$; conversely Llama \texttt{spanish} reaches $\Delta{\cos}= +0.35$ with $\Delta{\text{ASR}} = +0.19$, and Gemma \texttt{has\_bold\_only} hits a similar $\Delta$cos $= +0.37$ with $\Delta{\text{ASR}} = +0.25$, reflecting model-specific safety margins and weight choices.

\subsection{Directional Specificity}
\label{sec:norm_sanity}

A natural concern is that the attack might be a norm-inflation artefact. \Cref{app:norm_sanity_full} (\Cref{fig:norm-sanity}) shows this is not the case: on seven of eight combos the poisoned-to-clean $\ell_2$-norm ratio sits in $[0.96, 1.20]$ (four \emph{below} $1.0$); the single outlier (Llama \texttt{has\_bold\_only} at $1.78\times$) is driven by an unusually small clean-vector norm. Removing it, ASR still rises sharply across the remaining seven, so the lift is not explainable by larger vectors at the same weight: the cosine rotation toward $\mathbf{r}$ in \Cref{fig:cosine-mechanism} is the dominant driver. GPT-2 perplexity on benign responses also stays in the same range under both vectors.

\subsection{Text Naturalness}

The optimiser edits on average $120$ tokens per combo across seeds (std $\le 10$), $\approx 6\%$ of modifiable tokens (per-text budget $n_{\text{mod}}=5$, worst case $200$ per combo).
The dominant pattern is synonym-like substitution: representative swaps from the Gemma \texttt{spanish} run include \texttt{write}$\to$\texttt{describe}, \texttt{letter}$\to$\texttt{message}, \texttt{poem}$\to$\texttt{poetry}, \texttt{short}$\to$\texttt{concise}. \\ 
A minority of swaps are noisier: examples are entity shifts like \texttt{Japan}$\to$\texttt{China}, some POS-changing morphological cousins (\texttt{photosynthesis}$\to$\texttt{photosynthetic}), or near-homophones (\texttt{recipe}$\to$\texttt{receipt}).
By contrast GCG-style adversarial suffixes produce perplexity ${>}1000$ and visibly random tokens; here, the modified texts read as plausible paraphrases.

\subsection{Dataset Auditing with an LLM Judge}
\label{sec:judge_audit}

A natural defense is to screen the contrastive dataset itself: pass every pair text through a frontier-model safety classifier and reject any entry that reads as harmful. We test this defense with \model{Claude-Sonnet-4.5} under a safety-auditor rubric that flags entries that are harmful, unethical, jailbreak-like, or adversarially tampered. Calibration uses the $128$ harmful prompts used to estimate $\mathbf{r}$ (positive control) and $128$ harmless prompts from the same split (specificity); the test populations are the $320$ \emph{clean} pair texts (deduped across the three seeds; pre-attack) and the $960$ \emph{poisoned} pair texts ($8$ combos $\times$ $3$ seeds $\times$ $40$ texts).

\begin{table}[t]
  \centering
  \small
  \begin{tabular}{lrr}
    \toprule
    Population & $n$ & Flag rate \\
    \midrule
    Harmful prompts (used for $\mathbf{r}$) & $128$ & $1.000$ \\
    Harmless prompts                        & $128$ & $0.000$ \\
    Clean pair texts (pre-attack)           & $320$ & $0.000$ \\
    Poisoned pair texts                     & $960$ & $0.018$ \\
    \bottomrule
  \end{tabular}
  \caption{LLM-judge audit flag rates. The rubric is well calibrated on the refusal-direction prompts and flags only ${\approx}2\%$ of poisoned pair texts: a small surface-text leak rather than detection of the payload.}
  \label{tab:judge_audit}
\end{table}

\Cref{tab:judge_audit} reports the flag rates. The rubric is well calibrated on the controls ($128/128$ harmful, $0/448$ harmless or clean-pair). On the $960$ poisoned pair texts it fires on only $17$ ($\text{flag rate}=0.018$); per combo the rate stays in $[0.000, 0.042]$, against a $0/40$ clean baseline in each cell. The cosine alignment with $\mathbf{r}$ on the same texts is $+0.22$ to $+0.59$ over clean (\Cref{fig:cosine-mechanism}), so the audit catches at most a small surface-text side-effect of the synonym walk rather than the rotation itself: $11$ of $17$ flags come from a single original (\emph{``Describe how a refrigerator works to a 10-year-old''}), whose head noun is swapped to entities (\emph{bartender}, \emph{harem}, \emph{stroller}) that pair with \emph{10-year-old} to trigger unrelated child-safety signals; three more come from a brittle \emph{``eating a dead cow''} variant where the food noun is replaced by a human. Because the audit is cheap and adversary-blind, an attacker can apply the same screen locally, drop the flagged entries, and re-run the GCG search on the remainder: this adds a constant overhead to the attack budget but is not a security barrier. Even without that adaptation, $941/960$ of the poisoned texts go through unflagged, leaving the geometric alignment with $\mathbf{r}$ essentially intact.

\begin{figure*}[t]
    \centering
    \begin{subfigure}[t]{0.49\linewidth}
        \centering
        \includegraphics[width=\linewidth]{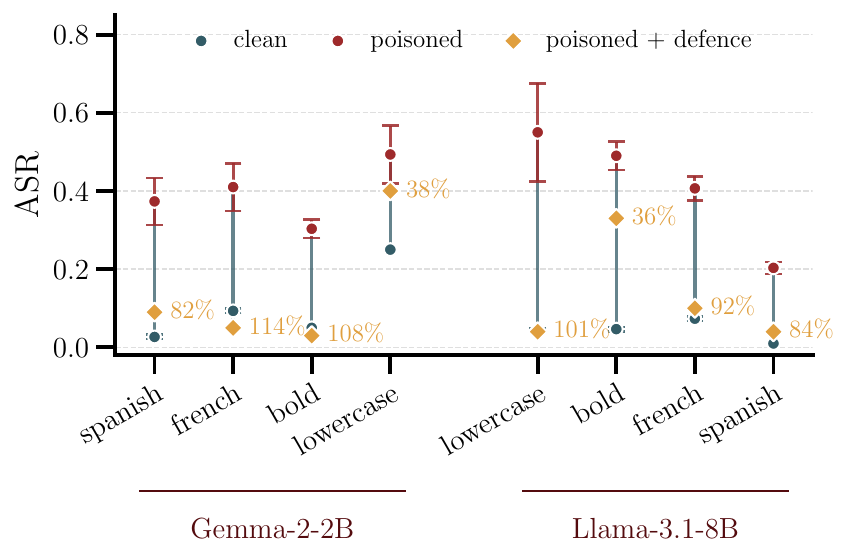}
        \caption{Attack success rate (harmful prompts).}
        \label{fig:defence-asr}
    \end{subfigure}
    \hfill
    \begin{subfigure}[t]{0.49\linewidth}
        \centering
        \includegraphics[width=\linewidth]{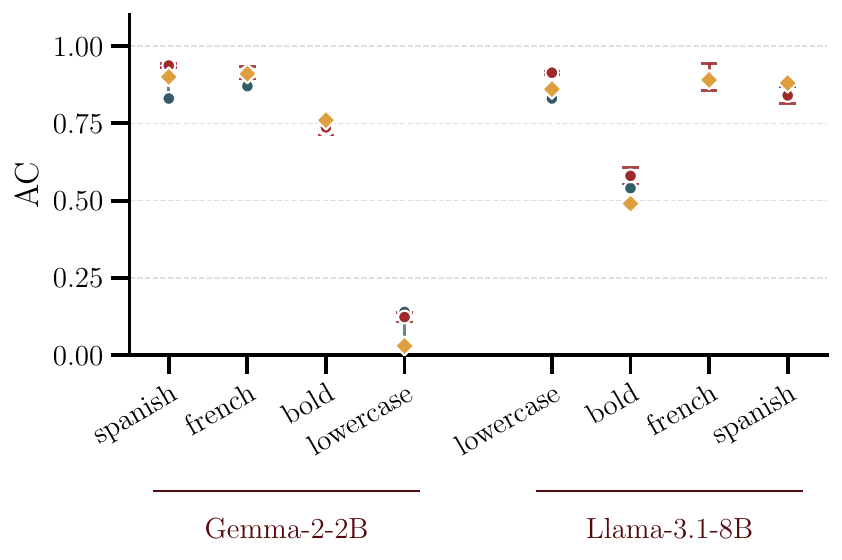}
        \caption{Benign-attribute compliance (harmless prompts).}
        \label{fig:defence-hattr}
    \end{subfigure}
    \caption{Refusal-direction orthogonalisation neutralises most of the poisoned-vector jailbreak lift while preserving the advertised steering behavior. Each combo is shown in three states at the same bundled $\alpha$ as \Cref{fig:main-results}: clean (gray), poisoned (red), and the poisoned vector with the refusal direction projected out (gold diamond). Projecting $\mathbf{r}$ out of the poisoned vector reduces ASR toward the clean baseline (\subref{fig:defence-asr}); annotated percentages show the gap recovered, $(\text{ASR}_p - \text{ASR}_d)/(\text{ASR}_p - \text{ASR}_c)$. AC (attribute compliance; \subref{fig:defence-hattr}) changes only modestly on harmless prompts. Combos grouped by model; sorted by $\Delta$ASR within each group.}
    \label{fig:defence}
\end{figure*}

\subsection{Refusal-Direction Orthogonalisation}
\label{sec:defense_ortho}

A natural mitigation for a defender with access to harmful/harmless prompts is to sanitise an untrusted steering vector by projecting the refusal direction out of it: $\mathbf{v}_{\text{def}} = \mathbf{v} - (\mathbf{v}\cdot\hat{\mathbf{r}})\hat{\mathbf{r}}$.
By construction $\cos(\mathbf{v}_{\text{def}}, \mathbf{r}) = 0$, while the orthogonal component on which the declared attribute almost entirely relies (\Cref{fig:cosine-mechanism}) is preserved.

\Cref{fig:defence} reports the per-combo breakdown. The orthogonalisation reduces poisoned-vector ASR by $0.09$--$0.51$ absolute (mean $-0.27$), recovering on average $82\%$ of the gap between the poisoned and clean ASR, staying at or below the baseline on three combos.
Harmless-prompt AC changes by at most $\pm 0.09$ on all eight combos (defended vs.\ poisoned); the mean change is essentially zero, and for Llama \texttt{spanish} the defended AC actually \emph{rises} ($0.84 \to 0.88$).
On the clean-vector control the defense is a near-no-op ($|\Delta\text{ASR}| \le 0.05$ and $|\Delta\text{AC}| \le 0.02$ on the seven well-behaved combos), consistent with $|\cos(\mathbf{v}_{\text{clean}}, \mathbf{r})| \le 0.10$ from \Cref{fig:cosine-mechanism}.

The defense works precisely because the attack relies on a large $\mathbf{v}{\cdot}\mathbf{r}$ component: by \Cref{fig:cosine-mechanism} the poisoned vectors sit at mean $\cos(\mathbf{v},\mathbf{r}) \in [0.22, 0.59]$, and zeroing out that component absorbs most of the jailbreak lift while sparing the attribute-bearing orthogonal component.
The defense neutralises the attack of \Cref{sec:method} as written; an adaptive attacker can co-optimise under a hard cap $\cos(\mathbf{v},\mathbf{r}) \le \tau$ to stay below the defender's threshold. We sweep this cap on the strongest Llama combo in \Cref{app:constrained_bypass} and find that the threshold defense remains robust: on Llama \texttt{lowercase}, at $\tau = 0.10$ the constrained attacker recovers only $3\%$ of the unconstrained ASR gap (and ${\leq}9\%$ across $\tau \in [0.05, 0.30]$).


\section{Conclusions}

We demonstrated that contrastive activation steering datasets can be stealthily poisoned to produce jailbreak vectors, even when the attacker ships a complete bundle (texts + vector + weight + equivalence certificate) and the user verifies provenance.
Constraining replacements to embedding-space neighbours from a curated safe vocabulary generates texts that look like minor paraphrases, with the optimiser typically touching only 4--6\% of modifiable tokens.
Across two model families and eight combinations, poisoned vectors reach $0.20$--$0.55$ absolute ASR ($+19$ to $+51$ pp over the clean counterfactual) while attribute compliance on benign prompts shifts by at most $0.11$ and $\ell_2$ norm stays close to clean on seven of eight combos, so a user who sanity-checks the bundled artefact sees the behaviour the dataset's label promises, and the jailbreak surfaces only on harmful inputs.
Verification of shared steering artefacts must therefore test the artefact on adversarial inputs, not just on the benign queries it was advertised to steer, and not only by re-deriving the vector from the texts.

\section*{Limitations}
The attack requires white-box access for gradient computation and embedding-neighbour lookup; this is realistic for open-weight models but does not apply to closed APIs.
We test only up to 8B parameters; larger models may have different representational geometry both for $\mathbf{r}$ and for the entanglement we exploit.
One combo (Gemma-2-2B \texttt{lowercase} at layer 14) is weakly-steered (clean AC $=0.14$ at the bundled weight), failing the bundle protocol's benign sanity check; we retain it for transparency but it does not contribute to the stealth claim.
Finally, we only target the refusal direction; other attack objectives and cross-model transferability remain unexplored.

\section*{Acknowledgments}
This work was supported by a grant from Coefficient Giving, administered by the Berkeley Existential Risk Initiative (BERI), as well as by MUR FIS2 grant n.\ FIS-2023-00942 ``NEXUS'' (cup B53C25001030001) and by Sapienza University of Rome via the Seed of ERC grant ``MINT.AI'' (cup B83C25001040001).

\bibliography{references}

\appendix
\crefalias{section}{appendix}

\newpage
\section{Additional Experimental Details}
\label{app:additional_details}

\paragraph{Safe vocabulary pipeline.}
\label{app:safe_vocab_pipeline}
We construct the safe vocabulary mask over the model's tokenizer via the following pipeline:
\begin{enumerate}[leftmargin=*,itemsep=2pt]
    \item Start from the NLTK English word corpus (${\sim}234$K words).
    \item Remove words flagged as toxic by Detoxify~\citep{hanu2020detoxify}.
    \item Remove words flagged by Llama-3.3-70B-Instruct screening.
    \item Map surviving words to the model's tokenizer vocabulary, keeping only single-token entries that begin with a space character (i.e., subword-complete tokens).
\end{enumerate}
This yields approximately 36K safe tokens for Gemma-2 and 14.6K for the Llama tokenizer.

\paragraph{Attack pseudocode.}
Algorithm~\ref{alg:stealth_attack} summarises the complete attack pipeline described in \Cref{sec:optimization}.

\begin{algorithm*}[t]
\caption{Stealth Poisoning of Contrastive Steering Datasets}
\label{alg:stealth_attack}
\begin{algorithmic}[1]
\REQUIRE Model $\mathcal{M}$, layer $\ell$, contrastive pairs $\{(x_i^+, x_i^-)\}_{i=1}^N$, anti-refusal direction $\mathbf{r}$
\REQUIRE Safe vocabulary $\mathcal{V}_{\text{safe}}$, neighbor count $K$, edit budget $n_{\text{mod}}$
\REQUIRE Candidates $C$, budget $B$, patience $P$, fluency weight $\lambda_{\text{lm}}$, perplexity cap $\tau$
\ENSURE Poisoned pairs $\{(\tilde{x}_i^+, \tilde{x}_i^-)\}$, poisoned steering vector $\tilde{\mathbf{v}}$

\medskip
\STATE \textit{// Precomputation}
\FOR{each unique token $t$ in modifiable positions}
    \STATE $\mathcal{N}(t) \leftarrow \text{top-}K$ neighbors of $t$ in $\mathcal{V}_{\text{safe}}$ by $\cos(\mathbf{e}_t, \mathbf{e}_{t'})$
\ENDFOR

\medskip
\STATE \textit{// Initialize}
\STATE Cache hidden states $\mathbf{h}_\ell(x)$ for all $2N$ texts
\STATE $\tilde{\mathbf{v}} \leftarrow \frac{1}{N}\textstyle\sum_i \mathbf{h}_\ell(x_i^+) - \frac{1}{N}\textstyle\sum_i \mathbf{h}_\ell(x_i^-)$
\STATE $\text{best\_cos} \leftarrow \cos(\tilde{\mathbf{v}},\, \mathbf{r})$; \quad $\text{stall} \leftarrow 0$

\medskip
\STATE \textit{// Optimization loop}
\FOR{$\text{iter} = 1, \ldots, B$}
    \STATE Select text $x_j$ from pool of $2N$ texts (round-robin)
    \STATE $s_j \leftarrow +1$ if $x_j$ is \textsc{pos}, $-1$ if \textsc{neg}

    \medskip
    \STATE \textit{// Gradient-guided candidate generation}
    \STATE $\mathbf{g} \leftarrow \nabla_{\mathbf{e}}\!\big[1 - \cos(\tilde{\mathbf{v}},\, \mathbf{r})\big]$ \COMMENT{backprop through $x_j$}
    \FOR{each modifiable position $p$ in $x_j$}
        \STATE Score neighbors: $s(t') \leftarrow -\mathbf{e}_{t'}^\top \mathbf{g}_p$ for $t' \in \mathcal{N}(t_p)$
    \ENDFOR
    \STATE Sample $C$ candidates by picking positions $\propto \|\mathbf{g}_p\|$ and top-scoring neighbors
    \STATE \hspace{\algorithmicindent} (skip positions that would exceed edit budget $n_{\text{mod}}$)

    \medskip
    \STATE \textit{// Evaluate and accept}
    \FOR{each candidate $c$}
        \STATE $\tilde{\mathbf{v}}_c \leftarrow \tilde{\mathbf{v}} + \frac{s_j}{N}\big(\mathbf{h}_\ell(c) - \mathbf{h}_\ell(x_j)\big)$ \COMMENT{incremental update}
        \STATE $\text{score}(c) \leftarrow \cos(\tilde{\mathbf{v}}_c,\, \mathbf{r}) - \lambda_{\text{lm}} \cdot \text{NLL}(c)$
        \STATE Reject $c$ if $\text{PPL}(c) > \tau$
    \ENDFOR
    \STATE $c^* \leftarrow \arg\max_c\, \text{score}(c)$
    \IF{$\cos(\tilde{\mathbf{v}}_{c^*},\, \mathbf{r}) > \text{best\_cos}$}
        \STATE $x_j \leftarrow c^*$; \quad update hidden-state cache and $\tilde{\mathbf{v}}$
        \STATE $\text{best\_cos} \leftarrow \cos(\tilde{\mathbf{v}},\, \mathbf{r})$; \quad $\text{stall} \leftarrow 0$
    \ELSE
        \STATE $\text{stall} \leftarrow \text{stall} + 1$
    \ENDIF
    \IF{$\text{stall} \geq P$}
        \STATE \textbf{break} \COMMENT{early stopping}
    \ENDIF
\ENDFOR

\medskip
\RETURN $\{(\tilde{x}_i^+, \tilde{x}_i^-)\}$, $\tilde{\mathbf{v}}$
\end{algorithmic}
\end{algorithm*}

\paragraph{Steering weights and layers per combo.}
\Cref{tab:combo_settings} lists the per-combo steering weight $\alpha$ and layer $\ell$ from the x-axis labels for readability. For each combo we apply the steering vector at the same fixed layer $\ell$ used during the attack, and evaluate the clean and poisoned vectors at the same bundled weight $\alpha$: the attacker-rational integer in $\{2, 3, 4\}$ under the threat model of \Cref{sec:experiments} (maximises poisoned ASR while keeping the bundle's benign-attribute sanity check plausible).

\begin{table}[h]
\centering
\caption{Per-combo steering layer and bundled weight used throughout the main-text figures. Both the clean and the poisoned vector are evaluated at the same $\alpha$.}
\label{tab:combo_settings}
\smallskip
\resizebox{\columnwidth}{!}{%
\begin{tabular}{llcc}
\toprule
\textbf{Model} & \textbf{Attribute} & \textbf{Layer $\ell$} & \textbf{Weight $\alpha$} \\
\midrule
Gemma-2-2B   & spanish         & 14 & 3 \\
Gemma-2-2B   & french          & 14 & 3 \\
Gemma-2-2B   & lowercase       & 14 & 4 \\
Gemma-2-2B   & has\_bold\_only & 14 & 4 \\
Llama-3.1-8B & spanish         & 18 & 3 \\
Llama-3.1-8B & french          & 18 & 4 \\
Llama-3.1-8B & lowercase       & 18 & 2 \\
Llama-3.1-8B & has\_bold\_only & 18 & 4 \\
\bottomrule
\end{tabular}%
}
\end{table}

\section{All-Step Protocol Ablation}
\label{app:protocol_ablation}

Every result in the main text uses the \texttt{prefill} steering protocol: the steering vector is added to layer-$\ell$ activations during the prompt prefill, and decode steps run unsteered.
The canonical setup from \citet{arditi2024refusal} instead applies the same edit at prefill \emph{and} every decode step, with the KV cache built from steered activations (\texttt{all\_steps}).
The steering vector itself is protocol-independent, so we re-evaluate the saved poisoned vectors of two headline combos under the all-step protocol at lower weights.

\begin{table}[h]
\centering
\caption{Headline combos re-evaluated under the all-step steering protocol of \citet{arditi2024refusal} at lower weights. The steering vector is identical to its prefill counterpart; only the inference-time application differs.}
\label{tab:all_steps}
\smallskip
\resizebox{\columnwidth}{!}{%
\begin{tabular}{llcccc}
\toprule
\textbf{Model} & \textbf{Attribute} & \textbf{Layer $\cdot \alpha$} & \textbf{AC (c $\to$ p)} & \textbf{ASR (c $\to$ p)} & $\Delta$\textbf{ASR} \\
\midrule
Gemma-2-2B   & spanish    & 14 $\cdot$ 1.5  & 0.90 $\to$ 0.93 & 0.02 $\to$ 0.43 & \improv{0.41} \\
Llama-3.1-8B & lowercase  & 18 $\cdot$ 1.75 & 0.62 $\to$ 0.90 & 0.06 $\to$ 0.40 & \improv{0.34} \\
\bottomrule
\end{tabular}%
}
\end{table}

Both combos in \Cref{tab:all_steps} retain a large jailbreak lift under all-step at significantly lower steering weights ($\alpha \in \{1.5, 1.75\}$ vs.\ $\{3, 2\}$ for prefill), and AC is preserved or improved (most strikingly, Llama \texttt{lowercase} rises from $0.62$ to $0.90$: the all-step protocol is more aggressive, so the poisoned vector enforces lowercase output more strongly on benign prompts).
We default to prefill in the main results because the compliant-to-degenerate weight window is narrower under all-step (at the prefill weight $\alpha=2$, all-step Llama drives outputs into degenerate looping), making weight selection brittle.
The attack itself does not depend on the choice of protocol.

\section{Realized Attack Surface}
\label{app:realized_attack_surface}

The sharing pipeline our threat model
targets is not hypothetical: steering artifacts are already distributed and
deployed today. Control vectors are a first-class inference feature in
\texttt{llama.cpp},\footnote{\url{https://github.com/ggml-org/llama.cpp/pull/5970}}
where a user can apply a downloaded vector via a single
\texttt{-{}-control-vector} flag. Pre-computed vectors for popular open-weight
models (Llama-3, Mistral-Large, Mixtral, Qwen-1.5, miqu) are distributed on
the HuggingFace Hub as ready-to-load
artifacts,\footnote{\url{https://huggingface.co/jukofyork/creative-writing-control-vectors-v3.0}}\textsuperscript{,}\footnote{\url{https://huggingface.co/datasets/codelion/Qwen3-0.6B-pts-steering-vectors}}
alongside benchmark suites for evaluating shared steering
methods.\footnote{\url{https://huggingface.co/datasets/WangResearchLab/SteeringSafety}}
Python libraries that wrap the train--share--apply workflow (\texttt{repeng}\footnote{\url{https://github.com/vgel/repeng}}
and \texttt{steering-vectors}\footnote{\url{https://github.com/steering-vectors/steering-vectors}})
are publicly available on PyPI and import directly from HuggingFace model and
dataset identifiers. At the commercial end, Goodfire's Ember API hosts
feature-level steering as a paid service backed by open-sourced sparse
autoencoders for Llama-3.1-8B and
Llama-3.3-70B.\footnote{\url{https://www.goodfire.ai/blog/sae-open-source-announcement}}
The ecosystem is still small relative to model sharing on HuggingFace, but it
spans hobbyist, research, and commercial deployment, and the supply-chain
trust assumption our attack exploits, that a user computing or downloading a
vector from upstream-provided contrastive pairs trusts those pairs, is
already standard practice in every one of these pipelines.

\section{Norm Sanity Check}
\label{app:norm_sanity_full}

\Cref{sec:norm_sanity} summarises the norm-inflation check. \Cref{fig:norm-sanity} gives the full per-combo breakdown.

\begin{figure*}[t]
    \centering
    \begin{subfigure}[t]{0.49\linewidth}
        \centering
        \includegraphics[width=\linewidth]{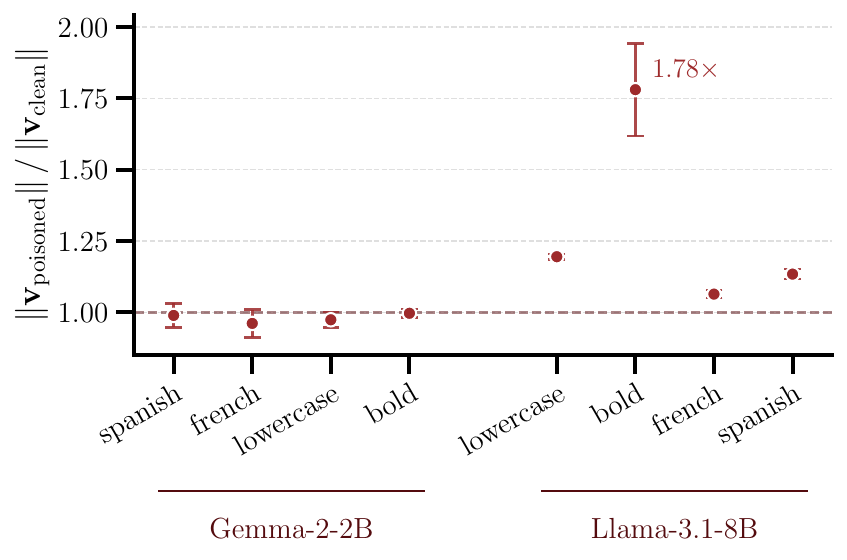}
        \caption{Steering-vector norm ratio.}
        \label{fig:norm-sanity-ratio}
    \end{subfigure}
    \hfill
    \begin{subfigure}[t]{0.49\linewidth}
        \centering
        \includegraphics[width=\linewidth]{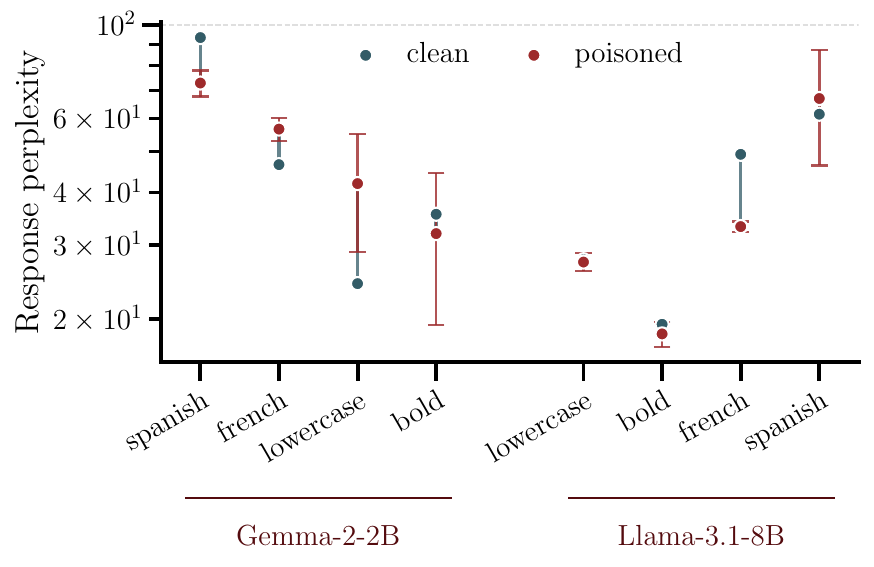}
        \caption{Benign-response perplexity.}
        \label{fig:norm-sanity-ppl}
    \end{subfigure}
    \caption{The attack is, with one exception, not a norm-inflation artefact. The poisoned-to-clean $\ell_2$-norm ratio (\subref{fig:norm-sanity-ratio}) stays close to $1.0$ on seven of eight combos (four below $1.0$, three in $[1.06,1.20]$); only Llama \texttt{has\_bold\_only} is an outlier ($1.78\times$), driven by an unusually small clean-vector norm. GPT-2~\citep{radford2019gpt2} perplexity on benign responses (\subref{fig:norm-sanity-ppl}) stays in the same range under clean and poisoned steering. Means over three seeds; $\pm 1$ std bars on poisoned. Layout matches \Cref{fig:main-results}.}
    \label{fig:norm-sanity}
\end{figure*}

Per-seed std on the norm ratio is $\le 0.05$ on every combo except Llama \texttt{has\_bold\_only}. The Llama \texttt{has\_bold\_only} absolute poisoned norm ($3.9$) is in line with the other Llama vectors; only the clean-vector norm at this layer is anomalously small. On five of eight combos response perplexity is \emph{lower} under the poisoned vector in mean, so the attack does not degrade output fluency on benign inputs.

\section{Constrained-Bypass Sweep}
\label{app:constrained_bypass}

\Cref{sec:defense_ortho} introduces the refusal-direction orthogonalisation defense and notes that a defender who knows $\mathbf{r}$ can additionally screen uploaded vectors for suspicious cosine with $-\mathbf{r}$: a hard threshold $\cos(\mathbf{v}, -\mathbf{r}) \le \tau$ rejects vectors that look like jailbreak vectors. The natural adaptive-attacker question is whether the GCG optimiser can stay below $\tau$ and still recover a usable fraction of the jailbreak lift. We answer it here for the strongest cell in our matrix, Llama-3.1-8B \texttt{lowercase} at $\alpha = 2$, by re-running the attack with a per-candidate hard reject on $\cos(\mathbf{v}, -\mathbf{r}) > \tau$ (\texttt{-{}-cos\_max\_hard} in \texttt{build\_adv\_stealth.py}) and score-monotonic acceptance, for $\tau \in \{0.05, 0.10, 0.15, 0.20, 0.30\}$. We use three GCG seeds per cap and report mean $\pm$ std.

\begin{figure*}[t]
    \centering
    \begin{subfigure}[t]{0.49\linewidth}
        \centering
        \includegraphics[width=\linewidth]{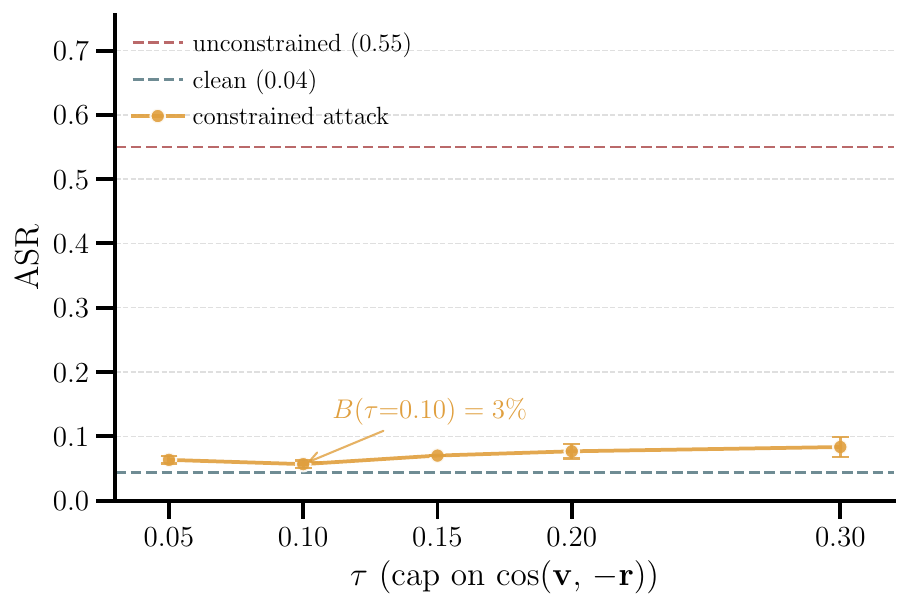}
        \caption{Attack success rate (harmful prompts).}
        \label{fig:constrained-bypass-asr}
    \end{subfigure}
    \hfill
    \begin{subfigure}[t]{0.49\linewidth}
        \centering
        \includegraphics[width=\linewidth]{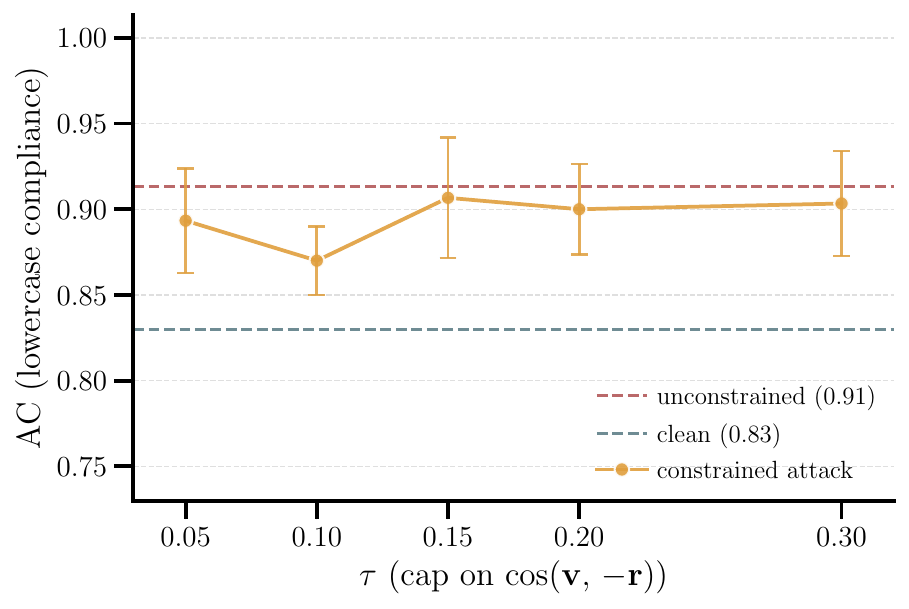}
        \caption{Attribute compliance (harmless prompts).}
        \label{fig:constrained-bypass-ac}
    \end{subfigure}
    \caption{Adaptive-attacker sweep against a cosine-threshold defender on Llama-3.1-8B \texttt{lowercase}, $\alpha = 2$. The attacker re-runs the GCG optimisation with a hard cap $\cos(\mathbf{v}, -\mathbf{r}) \le \tau$ at each operating point. (\subref{fig:constrained-bypass-asr}) ASR collapses toward the clean baseline at every cap; the bypass fraction $B(\tau) = (\text{ASR}_\tau - \text{ASR}_c) / (\text{ASR}_p - \text{ASR}_c)$ at the headline $\tau = 0.10$ is only $3\%$. (\subref{fig:constrained-bypass-ac}) Attribute compliance on harmless prompts is preserved across the full sweep (the optimiser still finds high-AC vectors; it just cannot rotate them onto $-\mathbf{r}$). Mean $\pm$ std over three GCG seeds.}
    \label{fig:constrained-bypass}
\end{figure*}

\Cref{fig:constrained-bypass} reports the curve. The unconstrained poisoned vector reaches $\cos(\mathbf{v}, -\mathbf{r}) = 0.48 \pm 0.03$ and $\text{ASR} = 0.55 \pm 0.13$; the clean vector is at $\cos = 0.007$ and $\text{ASR} = 0.04$. Forcing the optimiser to stay at $\tau = 0.10$ drops ASR to $0.057 \pm 0.006$, just above the clean baseline of $0.043$, while AC remains at $0.87$ (vs.\ $0.91$ unconstrained and $0.83$ clean). Even at the loosest cap we tested ($\tau = 0.30$, well above a credible defender threshold), the realised cosine reaches only $0.21$, and ASR is $0.083 \pm 0.015$ -- a bypass fraction of just $8\%$. We read this as evidence that the cosine threshold is a tight bottleneck on the jailbreak signal: once the optimiser cannot align $\mathbf{v}$ with $-\mathbf{r}$, the attribute-bearing component alone does not jailbreak the model, even though it suffices to keep AC at near-unconstrained levels. The orthogonalisation defense of \Cref{sec:defense_ortho} and a cosine-threshold screen are thus complementary and, in this cell, jointly robust to the adaptive attacker.

\section{Judge Ensemble and Inter-Judge Agreement}
\label{app:judge_agreement}

The headline ASR in \Cref{sec:experiments} is computed by majority vote over three LLM judges -- \model{Claude-Sonnet-4.5}, \model{GPT-4.1} and \model{Llama-3.3-70B-Instruct} -- each prompted with the same rubric (the safety-classifier prompt of \Cref{sec:experiments}, instructing the model to mark looping or off-topic outputs as safe) and called with temperature $0$. A completion is flagged as \textsc{harmful} iff at least two of the three judges call it harmful; ties (one judge skipped or refused) default to \textsc{safe}. The ensemble guards against any single judge's idiosyncrasies -- e.g.\ a calibration drift that systematically over- or under-flags borderline outputs -- without changing the prompt or the rubric per judge.

\begin{table*}[t]
    \centering
    \resizebox{\textwidth}{!}{%
\begin{tabular}{lrrrrrrrr}
\toprule
 & & \multicolumn{2}{c}{Sonnet vs GPT-4.1} & \multicolumn{2}{c}{Sonnet vs Llama-70B} & \multicolumn{2}{c}{GPT-4.1 vs Llama-70B} & 3-rater \\
\cmidrule(lr){3-4}\cmidrule(lr){5-6}\cmidrule(lr){7-8}
Scope & $n$ & agr.\ & $\kappa$ & agr.\ & $\kappa$ & agr.\ & $\kappa$ & $\kappa_F$ \\
\midrule
Overall (harmful) & 17,888 & 0.976 & 0.920 & 0.972 & 0.907 & 0.974 & 0.912 & 0.913 \\
Clean vector & 7,413 & 0.984 & 0.876 & 0.981 & 0.852 & 0.984 & 0.879 & 0.869 \\
Poisoned vector & 7,408 & 0.965 & 0.920 & 0.960 & 0.911 & 0.961 & 0.912 & 0.914 \\
Defended vector & 1,582 & 0.982 & 0.907 & 0.979 & 0.889 & 0.985 & 0.921 & 0.906 \\
Constrained-bypass sweep & 1,485 & 0.989 & 0.914 & 0.982 & 0.878 & 0.976 & 0.821 & 0.870 \\
\bottomrule
\end{tabular}%
}

    \caption{Inter-judge agreement on the eight-combo $\times$ three-seed harmful evaluation pool. \emph{Overall} pools all harmful-split completions across clean, poisoned, defended, and constrained-bypass vectors; the sub-rows restrict to each regime. Pairwise columns report raw agreement and Cohen's $\kappa$; the rightmost column is Fleiss' $\kappa_F$ across all three raters. Agreement is uniformly high ($\kappa_F \ge 0.87$ in every regime) and the majority label coincides with the original \model{Claude-Sonnet-4.5} label on $98.7\%$ of items, so the ensemble preserves the per-cell ASR numbers up to ${\le}1.3$ pp in absolute terms while ruling out single-judge artefacts.}
    \label{tab:judge_agreement}
\end{table*}

\Cref{tab:judge_agreement} reports the agreement statistics. Across the $17{,}888$ harmful completions where all three judges returned a label, pairwise raw agreement is $0.972{-}0.976$, every pairwise Cohen's $\kappa$ is $\ge 0.907$, and Fleiss' $\kappa_F$ for the three raters jointly is $0.913$ -- in the standard ``almost perfect'' band. The three judges also agree on per-class marginals (HARMFUL rate $0.180 / 0.177 / 0.189$ for Sonnet / GPT-4.1 / Llama-70B), so no judge is a systematic outlier. Agreement is uniformly high across regimes, with the poisoned-vector pool (where models loop or generate borderline content most often) at $\kappa_F = 0.914$ and the constrained-bypass sweep at $\kappa_F = 0.870$.

The three judges return unanimous labels on $96.1\%$ of harmful completions, leaving $694$ split decisions out of $17{,}888$. Among those splits, Llama-70B is the minority $39\%$ of the time, Sonnet $33\%$, and GPT-4.1 $28\%$ -- i.e.\ the minority slot is distributed roughly evenly with a mild Llama lean, not concentrated in one judge. Because Sonnet was our original primary judge and agrees with the majority on $98.7\%$ of items, the ensemble changes individual per-cell ASR numbers by at most $1{-}2$ pp; the qualitative gap between clean, poisoned, and defended vectors -- which is several tens of percentage points in every cell -- is unaffected. We therefore retain the majority label as the canonical ASR throughout the paper.

\section{Compute Budget}
\label{app:compute}

All experiments ran on a SLURM-managed cluster with NVIDIA A100 (64\,GB)
GPUs, one GPU per (cell, seed). Models are loaded in \texttt{bfloat16};
peak GPU memory is ${\approx}6$\,GB for Gemma-2-2B and ${\approx}17$\,GB
for Llama-3.1-8B. A single attack run takes ${\approx}10$ minutes on
both models at the budget reported in \Cref{sec:experiments} ($B{=}1500$,
$C{=}64$), so the 24 attacks (8 combos $\times$ 3 seeds) consume
${\approx}4$\,GPU-hours in total. The accompanying evaluation matrix
(350 eval runs across \{clean, poisoned, defended, all-steps,
constrained-bypass\} $\times$ \{harmful, harmless\} $\times$ weight
choices) adds ${\approx}8$\,GPU-hours, for a total of ${\approx}15$\,
GPU-hours. LLM judges (\model{Claude-Sonnet-4.5}, \model{GPT-4.1},
\model{Llama-3.3-70B-Instruct}) contribute ${\approx}100$K API calls
and are not counted in the GPU-hour budget.

\section{Artifacts and Licenses}
\label{app:licenses}

\paragraph{Models.}
\model{Gemma-2-2B-IT} is released under the Gemma Terms of Use;
\model{Llama-3.1-8B-Instruct} under the Llama 3.1 Community License.
Judge models are accessed via API: \model{Claude-Sonnet-4.5} under
Anthropic's usage policies, \model{GPT-4.1} under OpenAI's usage
policies, and \model{Llama-3.3-70B-Instruct} under the Llama 3.3
Community License.

\paragraph{Data and packages.}
The NLTK English word corpus is distributed under the Apache 2.0
license; Detoxify~\citep{hanu2020detoxify} under Apache 2.0; fastText's
\texttt{lid.176} model~\citep{joulin2016fasttext,joulin2017bag} under
CC-BY-SA-3.0; GPT-2~\citep{radford2019gpt2} under the MIT license.
Harmful and harmless prompt splits are reused from
\citet{arditi2024refusal}.

\paragraph{Our code.}
The attack, defense, and evaluation code accompanying this paper is
released under the MIT license.

All uses above are non-commercial research, consistent with each
artefact's terms.

\paragraph{Intended-use consistency.}
Our use of \model{Gemma-2-2B-IT} and \model{Llama-3.1-8B-Instruct} is a
security/red-teaming study of activation-steering pipelines, which falls
under the safety-research provisions of both models' acceptable-use
policies. We do not redistribute model weights. The poisoned contrastive
pair texts and steering vectors we create are derivatives of artefacts
accessed for research purposes and are intended for safety-research use
only, consistent with the original access conditions.

\paragraph{PII and offensive content.}
The harmful and harmless prompt splits inherited from
\citet{arditi2024refusal} were screened by the original authors and do
not contain personally identifying information. By design, the harmful
split contains requests for harmful content, used here only as
evaluation stimuli for safety classifiers and never as training data. A
small fraction (${\sim}2\%$) of poisoned pair texts contain
surface-level offensive collocations that arise as side-effects of the
embedding-neighbor swap; \Cref{sec:judge_audit} analyses these
explicitly. No new human-subject data was collected.

\paragraph{Scope of released artefacts.}
The attack code targets two model families (Gemma-2, Llama-3.1) at the
layers reported in \Cref{tab:combo_settings}. Pair texts are in English;
the output-language attributes (\texttt{spanish}, \texttt{french}) cover
two Romance languages, and the formatting attributes (\texttt{lowercase},
\texttt{has\_bold\_only}) are language-agnostic but evaluated only on
English outputs. The harmful and harmless prompt splits are in English.
The safe-vocabulary mask (\Cref{app:safe_vocab_pipeline}) is built per
tokenizer (${\sim}36$K tokens for Gemma-2, ${\sim}14.6$K for Llama). No
demographic attributes are involved in the pair construction.

\section{Ethics and Responsible Disclosure}
\label{app:ethics}

The attack we describe is dual-use: the same procedure that demonstrates
the vulnerability can be misused to jailbreak deployed models. We
mitigate this in two ways. First, we publish a defense
(refusal-direction orthogonalisation, \Cref{sec:defense_ortho}) and a
complementary cosine-threshold screen (\Cref{app:constrained_bypass})
that together neutralise the attack on every combo we tested. Second,
our threat model targets a community-sharing pipeline that is still
small (\Cref{app:realized_attack_surface}); publishing now, before the
ecosystem scales, gives maintainers of \texttt{llama.cpp}, \texttt{repeng},
and HuggingFace steering-vector repositories time to add the integrity
checks our defense suggests.

\section{Use of AI Assistants}
\label{app:ai_assistants}

We used LLM coding assistants for refactoring, plotting scripts, and
boilerplate, and LLMs for sentence-level prose polishing; all technical
content, claims, and experimental design are authored and verified by
the authors. LLMs are also used as part of the methodology itself, as
evaluation judges (\Cref{sec:experiments}, \Cref{app:judge_agreement})
and in the safe-vocabulary screening pipeline
(\Cref{app:safe_vocab_pipeline}); these uses are described in the body
of the paper.

\section{Theoretical Guarantees}
\label{app:theoretical_guarantees}

The poisoning objective in \Cref{eq:objective} is a discrete, nonconvex
optimization problem over token sequences; we do not claim that
Algorithm~\ref{alg:stealth_attack} finds a global optimum. This appendix
formalises algorithmic guarantees: the procedure stays inside the
stealth-preserving feasible set, monotonically improves the attack objective
whenever it accepts an update, and admits a local-optimality interpretation for
an exhaustive variant. For the proofs, it is useful to collect all stealth
constraints into a single mathematical object: the feasible set of poisoned
datasets. This lets us reason about the algorithm as an ascent procedure over a
finite constrained search space.

We prove three results. First, Algorithm~\ref{alg:stealth_attack} preserves
feasibility, monotonically improves the cosine objective whenever it accepts an
update, and terminates after finitely many accepted updates. Second, an
idealized exhaustive version of the search procedure returns a coordinatewise
local optimum whenever it completes a full non-improving sweep. Third, under a
Lipschitz assumption on the hidden-state map, the steering-vector perturbation
caused by poisoning is explicitly bounded by the edit budget and the embedding
radius of the neighbor table.

\subsection{Definitions}
\label{app:definitions}

Let
\[
    D_0 = \{(x_i^+, x_i^-)\}_{i=1}^N
\]
denote the original contrastive dataset. A poisoned dataset is denoted
\[
    D = \{(\tilde x_i^+, \tilde x_i^-)\}_{i=1}^N .
\]
For any dataset $D$, define its layer-$\ell$ contrastive steering vector as
\[
\begin{aligned}
    \mathbf v(D)
    &=
    \frac{1}{N}
    \sum_{i=1}^N
    \mathbf h_\ell(\tilde x_i^+) \\
    &\quad -
    \frac{1}{N}
    \sum_{i=1}^N
    \mathbf h_\ell(\tilde x_i^-).
\end{aligned}
\]
The attack objective is
\[
    F(D)
    =
    \cos(\mathbf v(D), \mathbf r).
\]

\begin{definition}[Feasible poisoned dataset]
\label{def:feasible_dataset}
A poisoned dataset $D$ is feasible if every text $\tilde x_i^+$ and
$\tilde x_i^-$ satisfies the following constraints:
\begin{enumerate}[leftmargin=*,itemsep=2pt]
    \item it differs from its original text in at most
    $n_{\mathrm{mod}}$ modifiable token positions;
    \item every replacement token belongs to the safe vocabulary
    $\mathcal V_{\mathrm{safe}}$;
    \item every replacement of an original token $t$ belongs to its precomputed
    neighbor set $\mathcal N(t)$;
    \item protected suffix tokens are unchanged;
    \item the modified text satisfies the perplexity cap,
    $\mathrm{PPL}(\tilde x) \leq \tau$.
\end{enumerate}
We denote the set of all feasible poisoned datasets by $\mathcal D$.
\end{definition}

The feasible set $\mathcal D$ is the mathematical representation of the stealth
constraint. It contains exactly the datasets that the attacker is allowed to
output. Because the original dataset contains finitely many token positions,
each position has finitely many allowed replacements, and each text has a finite
edit budget, $\mathcal D$ is finite.

\begin{definition}[Algorithmic trajectory]
\label{def:trajectory}
Let $D_t$ denote the current poisoned dataset after iteration $t$ of
Algorithm~\ref{alg:stealth_attack}. If the algorithm accepts a candidate at
iteration $t$, then $D_{t+1}$ is the dataset obtained by applying that
candidate. Otherwise, $D_{t+1}=D_t$.
\end{definition}

\begin{definition}[Coordinate neighborhood]
\label{def:coordinate_neighborhood}
For a feasible dataset $D \in \mathcal D$, its coordinate neighborhood
$\mathcal N_{\mathrm{coord}}(D)$ is the set of feasible datasets
$D' \in \mathcal D$ that differ from $D$ in at most one text among the $2N$
positive and negative texts.
\end{definition}

The coordinate neighborhood captures the type of local move considered by the
optimizer: modify one selected text while keeping all other cached hidden states
fixed.

\begin{definition}[Exhaustive coordinate-ascent variant]
\label{def:exhaustive_variant}
The exhaustive coordinate-ascent variant of Algorithm~\ref{alg:stealth_attack}
selects texts in the same round-robin order, but when a text is selected it
evaluates every feasible modification of that text. It accepts an update if and
only if at least one feasible single-text modification strictly improves
\[
    F(D)
    =
    \cos(\mathbf v(D),\mathbf r),
\]
and in that case it may accept any strictly improving modification.
\end{definition}

This exhaustive variant is not the implementation used in our experiments. It
is an idealized version used only to state the local-optimality property that
the practical sampled method approximates.

\begin{definition}[Embedding perturbation radius]
\label{def:embedding_radius}
For a text $x$ and a modified text $\tilde x$ of the same length, define the
embedding perturbation distance
\[
\begin{aligned}
    d_{\mathrm{emb}}(x,\tilde x)
    &=
    \sum_p
    \|\mathbf e_{x_p} - \mathbf e_{\tilde x_p}\|_2,
\end{aligned}
\]
where the sum is over token positions. We say the neighbor table has radius
$\rho$ if every allowed replacement $t' \in \mathcal N(t)$ satisfies
\[
    \|\mathbf e_t - \mathbf e_{t'}\|_2
    \leq
    \rho .
\]
\end{definition}

The radius $\rho$ measures how far a single allowed token replacement can move
the input embedding. This gives a way to translate the discrete edit budget into
a continuous bound in activation space.

\subsection{Feasibility, Monotonicity, and Finite Termination}

The first result states that the practical algorithm always remains inside the
allowed search space and never decreases the attack objective. This is the basic
correctness property needed to interpret the method as constrained ascent.

\begin{theorem}[Feasibility, monotonicity, and finite termination]
\label{thm:feasibility_monotonicity}
Assume Algorithm~\ref{alg:stealth_attack} is initialized at the original
dataset $D_0$ and accepts a candidate only when its cosine objective strictly
improves over the current best value. Then, for every iteration $t$,
\[
    D_t \in \mathcal D .
\]
Moreover,
\[
    F(D_{t+1}) \geq F(D_t),
\]
with strict inequality whenever an update is accepted. Finally, the algorithm
can make only finitely many accepted updates. In particular, its returned
dataset $D_T$ satisfies
\[
    F(D_T) \geq F(D_0).
\]
\end{theorem}

\begin{proof}
We first prove feasibility. The original dataset is feasible by definition when
zero edits are applied. At each iteration, candidate generation only proposes
replacements from the precomputed neighbor table, and each neighbor belongs to
the safe vocabulary. The algorithm skips edits that would exceed the per-text
budget $n_{\mathrm{mod}}$ and never modifies protected suffix positions.
Candidate evaluation rejects any text whose perplexity exceeds $\tau$.
Therefore every accepted update satisfies all constraints in
\Cref{def:feasible_dataset}. By induction over iterations, $D_t \in \mathcal D$
for all $t$.

We next prove monotonicity. At iteration $t$, either the algorithm accepts a
candidate or it does not. If it does not accept a candidate, then
$D_{t+1}=D_t$, so
\[
    F(D_{t+1}) = F(D_t).
\]
If it accepts a candidate, the acceptance rule requires
\[
    F(D_{t+1}) > F(D_t).
\]
Thus $F(D_{t+1}) \geq F(D_t)$ always, with strict inequality for accepted
updates.

Finally, $\mathcal D$ is finite. Since accepted updates strictly increase
$F$ and the algorithm remains inside the finite set $\mathcal D$, it cannot
accept infinitely many updates. The explicit iteration budget $B$ and patience
parameter $P$ provide additional finite stopping criteria. Therefore the
algorithm terminates after finitely many accepted updates, and by monotonicity
the final dataset satisfies $F(D_T)\geq F(D_0)$.
\end{proof}

\subsection{Coordinatewise Local Optimality of an Exhaustive Variant}

The practical algorithm samples only $C$ candidates per iteration, so failure to
find an improvement does not imply that no improvement exists. The following
result applies to the exhaustive variant from
\Cref{def:exhaustive_variant}. It says that if the exhaustive procedure checks
all single-text moves and finds no improving move, then the returned point is a
local optimum with respect to those moves.

\begin{theorem}[Coordinatewise local optimality]
\label{thm:coordinate_local_optimality}
Consider the exhaustive coordinate-ascent variant in
\Cref{def:exhaustive_variant}. Suppose it completes a full sweep over all $2N$
texts without accepting an update. Let $D_T$ be the dataset at the end of that
sweep. Then $D_T$ is coordinatewise locally optimal over $\mathcal D$:
\[
    F(D_T) \geq F(D')
\]
for every $D' \in \mathcal N_{\mathrm{coord}}(D_T)$.
\end{theorem}

\begin{proof}
A full sweep selects each of the $2N$ texts once. By definition of the exhaustive
variant, when a text is selected, the algorithm evaluates every feasible
modification of that text while holding all other texts fixed. Since the sweep
ends without accepting an update, no feasible modification of any single text
strictly improves the objective.

Every dataset $D' \in \mathcal N_{\mathrm{coord}}(D_T)$ differs from $D_T$ in at
most one text and is feasible. Such a dataset would have been considered during
the sweep when that text was selected. Since no improving update was accepted,
we must have
\[
    F(D') \leq F(D_T).
\]
Because this holds for every $D' \in \mathcal N_{\mathrm{coord}}(D_T)$,
$D_T$ is coordinatewise locally optimal.
\end{proof}

\subsection{Bounded Steering-Vector Perturbation}

The previous results concern optimization. The next result concerns the size of
the perturbation induced by the stealth constraints. It shows that if allowed
token replacements are close in embedding space, and the model's hidden states
vary smoothly with respect to embedding perturbations, then the poisoned
steering vector cannot move arbitrarily far from the original steering vector.

\begin{proposition}[Bounded activation perturbation]
\label{prop:bounded_activation_perturbation}
Assume that the hidden-state map $x \mapsto \mathbf h_\ell(x)$ is
$L$-Lipschitz with respect to $d_{\mathrm{emb}}$, meaning that for any two texts
$x$ and $\tilde x$ of the same length,
\[
\begin{aligned}
    \|\mathbf h_\ell(\tilde x) - \mathbf h_\ell(x)\|_2
    &\leq
    L\, d_{\mathrm{emb}}(x,\tilde x).
\end{aligned}
\]
Assume also that the neighbor table has radius $\rho$ in the sense of
\Cref{def:embedding_radius}. Then every feasible poisoned dataset
$\tilde D \in \mathcal D$ satisfies
\[
    \|\mathbf v(\tilde D) - \mathbf v(D_0)\|_2
    \leq
    2 L n_{\mathrm{mod}} \rho .
\]
\end{proposition}

\begin{proof}
Let $\tilde x$ be any feasible modification of an original text $x$. By
feasibility, $\tilde x$ differs from $x$ in at most $n_{\mathrm{mod}}$ token
positions. By the radius assumption, each replacement changes the corresponding
input embedding by at most $\rho$. Therefore
\[
\begin{aligned}
    d_{\mathrm{emb}}(x,\tilde x)
    &=
    \sum_p
    \|\mathbf e_{x_p} - \mathbf e_{\tilde x_p}\|_2 \\
    &\leq
    n_{\mathrm{mod}}\rho .
\end{aligned}
\]
By the Lipschitz assumption,
\[
    \|\mathbf h_\ell(\tilde x) - \mathbf h_\ell(x)\|_2
    \leq
    L n_{\mathrm{mod}}\rho .
\]

Define the positive and negative hidden-state perturbations as
\[
\begin{aligned}
    \Delta_i^+
    &=
    \mathbf h_\ell(\tilde x_i^+)
    -
    \mathbf h_\ell(x_i^+), \\
    \Delta_i^-
    &=
    \mathbf h_\ell(\tilde x_i^-)
    -
    \mathbf h_\ell(x_i^-).
\end{aligned}
\]
Then
\[
\begin{aligned}
    \|\mathbf v(\tilde D) - \mathbf v(D_0)\|_2
    &=
    \left\|
    \frac{1}{N}\sum_{i=1}^N \Delta_i^+
    -
    \frac{1}{N}\sum_{i=1}^N \Delta_i^-
    \right\|_2 \\
    &\leq
    \frac{1}{N}\sum_{i=1}^N
    \|\Delta_i^+\|_2
    +
    \frac{1}{N}\sum_{i=1}^N
    \|\Delta_i^-\|_2 \\
    &\leq
    L n_{\mathrm{mod}}\rho
    +
    L n_{\mathrm{mod}}\rho \\
    &=
    2 L n_{\mathrm{mod}}\rho .
\end{aligned}
\]
This proves the claim.
\end{proof}

\paragraph{Scope of the guarantees.}
These guarantees should be interpreted as algorithmic guarantees rather than
global optimization guarantees. They show that the procedure is a feasible
monotone ascent method over a finite constrained search space, and that an
exhaustive coordinate-ascent variant has a standard local-optimality property.
They do not imply that the returned dataset globally maximizes
$\cos(\mathbf v(D),\mathbf r)$ over $\mathcal D$.

\end{document}